\RequirePackage{fix-cm}
\documentclass[smallextended]{svjour3}       % onecolumn (second format)
\usepackage{etex}
\reserveinserts{28}
\smartqed  % flush right qed marks, e.g. at end of proof
\usepackage{fancyhdr} % Required for custom headers
\usepackage{lastpage} % Required to determine the last page for the footer
\usepackage{extramarks} % Required for headers and footers
\usepackage{graphicx} % Required to insert images
\usepackage{lipsum} % Used for inserting dummy 'Lorem ipsum' text into the template
\usepackage{authblk}

\usepackage{multirow}
\usepackage{multicol}
\usepackage{mathrsfs}

\usepackage{amsmath}
\usepackage{amsthm}
\usepackage{times}  %Required
\usepackage{helvet}  %Required
\usepackage{courier}  %Required
\usepackage{url}  %Required
\usepackage{graphicx}  %Required
\usepackage{helvet}
\usepackage{courier}
\usepackage{graphicx}
\usepackage{mathrsfs}
\usepackage[ruled,vlined,lines numbered]{algorithm2e}
\usepackage{subfigure}
\usepackage{enumerate}
\usepackage{caption}

\usepackage{tikz}
\usepackage{bbold}
\usepackage{colortbl}
\usepackage{pgfplots}

\usepackage[utf8]{inputenc}
\usepackage[english]{babel}
\usepackage[nottoc]{tocbibind}

\usepackage[square, comma, sort&compress ,numbers]{natbib}
\begin{document}

\title{Multi-Relevance Transfer Learning%\thanks{Grants or other notes
}


\author{Tianchun Wang         \and
        Xiaoming Jin           \and
        Xiaojun Ye %etc.
}


\institute{Tianchun Wang \at
              School of Information Systems\\
              Singapore Management University \\
              80 Stamford Road, Singapore 178902 \\
             % Tel.: +123-45-678910\\
             % Fax: +123-45-678910\\
              \email{tcwang@smu.edu.sg}           %  \\
           \and
           Xiaoming Jin, Xiaojun Ye \at
              School of Software, Tsinghua University\\
              Beijing 100084, China\\
              \email{\{xmjin,xjye\}@tsinghua.edu.cn}
}

\date{Received: date / Accepted: date}

\maketitle

\begin{abstract}
Transfer learning aims to faciliate learning tasks in a label-scarce target domain by leveraging knowledge from a related source domain with plenty of labeled data. Often times we may have multiple domains with little or no labeled data as targets waiting to be solved. Most existing efforts tackle target domains separately by modeling the `source-target' pairs without exploring the relatedness between them, which would cause loss of crucial information, thus failing to achieve optimal capability of knowledge transfer. In this paper, we propose a novel and effective approach called Multi-Relevance Transfer Learning (MRTL) for this purpose, which can simultaneously transfer different knowledge from the source and exploits the shared common latent factors between target domains. Specifically, we formulate the problem as an optimization task based on a collective nonnegative matrix tri-factorization framework. The proposed approach achieves both source-target transfer and target-target leveraging by sharing multiple decomposed latent subspaces. Further, an alternative minimization learning algorithm is developed with convergence guarantee. Empirical study validates the performance and effectiveness of MRTL compared to the state-of-the-art methods.
\end{abstract}

\section{Introduction}
\label{intro}
Transfer learning, which intends to utilize knowledge from source domains to help the learning in a target domain, has been established as one of the most important machine learning paradigms~\cite{pan2010survey}. In practice, a common scenario is that test data are often sampled from different distributions. One example is the EEG-based Brain Computer Interfaces (BCI) applications. If people want to classify the EEG data collected from several sessions (e.g. more than one hour) while only one session of them are labeled, then the unlabeled sessions can be seen as the target domains and the labeled one is source domain.  In this case, the distribution divergences between the source domain and different target domains may vary widely. Another characteristic is that the target domains may share some common latent structure which can help enhance the knowledge transfer.  Hence, a significant requirement for sufficient transfer learning in this scenario is to simultaneously exploit the relatedness between target domains and borrow different knowledge from the source
domain to each target domain.  \\
\indent However, most existing domain adaptation methods are designed for transferring knowledge from one or multiple source domains to a single target domain. We refer to such approach as single-relevance transfer learning. That is, the information path only between source and target. These methods do not consider the underlying relatedness between target domains.  Incurred by the multi-domain property, learning one target domain can help to learn another. It will lead to mutual reinforcement when learning the target domains together. Without exploiting the relatedness between targets, existing methods may only seperately transfer the common knowledge in each `source-target' pairs, which may result in partial transfer and is difficult to achieve optimal capability of knowledge transfer.
To exploit the relatedness between domains, multi-task learning~\cite{evgeniou2007multi,dredze2010multi} is a good choice which tackles these related tasks together by extracting and utilizing appropriate shared information across domains.  However, multi-task learning techniques are suitable for the cases that training and test data in each domain are sampled from the same distribution and each domain has reasonably large amounts of labeled data. Therefore, these methods would fail to transfer different knowledge to each target domain from the source, which are not ideal for such applications.\\
%If applied directly, these methods would fail to transfer different knowledge to each target domain from the source. Therefore, they are not ideal for such applications. \\
\indent In this paper, we propose a novel approach, Multi-Relevance Transfer Learning (MRTL), which simultaneously transfers different knowledge from the source to each target domain, and exploits the relatedness between targets to achieve knowledge reinforcement. The main idea of MRTL is illustrated in Figure 1. Different from traditional single-relevance methods, MRTL enhances transfer capability by targets exploration. More specifically, MRTL is formulated as an optimization problem of collective nonnegative matrix tri-factorization (NMTF). It decouples the source domain feature into multiple shared latent subspaces as bridges for
source-target knowledge transfer and subspaces of remaining feature clusters in each domain. Moreover, the target domains share a cluster association subspace to enable mutual reinforcement. We develop an alternating learning algorithm to optimize the objective. We give theoretical analysis of the proposed algorithm for convergence, and empirically show the effectiveness of the proposed method.
Overall, our main contributions of this paper include: (1) In addressing multi-relevance transfer learning problem, we propose a MRTL framework to achieve both source-target transfer and target-target transfer by exploiting shared feature subspaces; (2) We develop an alternating algorithm for optimization; (3) We analyze the theoretical convergence guarantee of the proposed algorithm, and also examine their empirical performances extensively.
\begin{figure}[htbp]
\begin{center}
\includegraphics[width=0.75\linewidth]{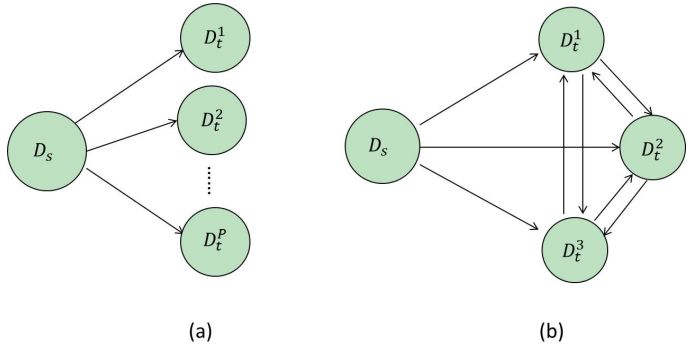}
\caption{(a) shows traditional single-relevance transfer learning. Knowledge is transferred from source to one target domain each time. (b) is Multi-Relevance Transfer Learning (MRTL). Knowledge simultaneously comes from multiple domains: from source domain and other target domains.}
\end{center}
\end{figure}

\section{Related Work and Preliminaries}
In this section, we first discuss several prior researches that mostly related to our work. Then we introduce the NMTF framework as preliminary.\\
%This section discusses three related categories: transfer learning, multi-task learning, and %semi-supervised learning.\\
\indent \textbf{Transfer Learning} solves the training data and test data obtained from different resources with different distributions.  Most existing efforts assume that there is shared knowledge structure acting as a bridge between the source domain and target domain to enable knowledge transfer.  Existing approaches can be grouped into two categories: instance-based transfer learning and feature-based transfer learning \cite{pan2010survey}. Instance-based transfer learning use re-weighting strategy to adapting the weights of source domain data~\cite{gretton2009covariate,DBLP:conf/acl/JiangZ07}. The second are feature representation methods which aim to learn a shared feature space to embedding the cross-domain feature information~\cite{blitzer2006domain,Blitzer11domainadaptation,DBLP:conf/ijcai/PanTKY09,gupta2010nonnegative,DBLP:conf/kdd/TanSZ015}.  Existing feature representation methods focus on transfer single-relevance latent structure from source to target. For example, Dual Transfer Learning (DTL)~\cite{DBLP:conf/sdm/LongWDCZW12} aims to simultaneously learning the marginal and conditional distributions across domains, Triplex Transfer Learning (TriTL)~\cite{zhuang2014triplex} which make source and target domain share one set of latent subspaces to transfer information. Although their formulations can also be extended to multiple target domains by sharing the same latent structures, it would fail to transfer different knowledge to each target from the source. The key difference between MRTL and these previous methods is that MRTL simultaneously learns different latent subspaces as bridges for knowledge transfer in each `source-target' pair, which is a crucial step to enhance the transfer diversity and capability.

More recently, multi-source transfer learning~\cite{ZhangGS2015,DuanTNNLS2012} have been developed to combine knowledge from multiple sources. For instance, the work in~\cite{DBLP:journals/tkdd/ChattopadhyaySFDPY12} present a two-stage domain adaptation method which combines weighted data from multiple sources. The study in~\cite{DuanTNNLS2012} propose a framework which can learn a robust decision function for label prediction for knowledge transfer.

Different from previous transfer learning approaches, multi-relevance transfer learning does not assume that auxiliary knowledge should only come from the source domain. That means, multi-relevance transfer learning can be more general and useful when the existing labeled auxiliary domain is not
adequate enough to improve the target domains.\\
\indent \textbf{Multi-task Learning} approaches simultaneously learn several
tasks together to mutual reinforcement the classification results of
each task~\cite{obozinski2006multi,evgeniou2007multi,zhang2012multi,dredze2010multi,cheng2013flexible}. It assumes that different tasks may share some common pattern, such as data clusters or subspaces. In practice, classifiers for different tasks can be designed to share some global parameters~\cite{DBLP:conf/kdd/EvgeniouP04} or even a global classifier~\cite{chapelle2010multi}. However, these methods require reasonably large amounts of labeled data in each domain to learn the relationship. In contrast, multi-relevance transfer learning works even when all the learning tasks in each target domain have no available ground truth. It only assumes that the source domain should have sufficient label information to transfer.\\
\indent \textbf{Non-negative Matrix Tri-Factorization (NMTF)} is popular and effective for data clustering and classification~\cite{DBLP:conf/kdd/DingLPP06}. It can decompose the feature-instance matrix into three submatrices. In general, given feature-instance matrix $X \in \mathbb{R}^{M \times N}$, $M$ is dimensionality and $N$ is instance number. One can obtain the factorized submatrices by solving the optimization problem given by:
\begin{align}
\min_{U,\Theta,V^{\mathrm T}} \mathcal{L}= \Arrowvert X-U \Theta V^{\mathrm T} \Arrowvert ^2,
\label{eq:NMTF}
\end{align}
where $\Arrowvert \cdot \Arrowvert$ is Frobenius norm of matrix. The matrix $U \in \mathbb{R}^{M \times k}$ indicates \emph{feature cluster subspace} and $k$ is the number of clusters in totoal. $U_{ij}$ is the probability that the $i$-th feature belongs to the $j$-th feature cluster. The matrix $V \in \mathbb{R}^{N \times c}$ is the instance \emph{cluster assignment} matrix and $c$ is cluster number. Let $V_{i,\tau} = \max_{1\leq j \leq c} V_{i, j}$, it means that the $i$-th instance belongs to the $\tau$-th cluster. For classification, each instance cluster can be regarded as a label class. The matrix $\Theta \in \mathbb{R}^{k \times c}$ is the \emph{cluster association} subspace. $\Theta_{ij}$ indicates the probability that the $i$-th feature cluster is associated with the $j$-th instance cluster.

\section{Multi-Relevance Knowledge Transfer}
We focus on transductive transfer learning where the source domain has abundant labeled examples while the target domains have unlabeled data. We consider one source domain $\mathcal{D}_s$ and multiple target domains $\mathcal{D}_t^p$, $p=1,2,...,P$.  $\mathcal{D}_s$ and $\mathcal{D}_t^p$ share the same feature dimensionality and label space. Here we consider $M$ features and $c$ classes. Let $X_s = \left[x^s_1,...,x^s_{n_s} \right] ^{\mathrm T}\in \mathbb{R}^{M\times n_{s}}$ represents the
feature-instance matrix of source domain $\mathcal{D}_{s}$, while $X_t^p = [x^p_1,...,x^p_{n_t^p}] ^{\mathrm T} \in \mathbb{R}^{M\times n_{t}^p}$ denotes the feature-instance matrix of the $p$-th target domain $\mathcal{D}_t^p$.  Labels in the source domain $\mathcal{D}_{s}$ are given as $Y_{s} \in \mathbb{R}^{n_{s} \times c}$, where $y_{ij}^s=1$ if $x_{i}$ belongs to class $j$, and $y_{ij}^s=0$ otherwise. Given $\left\{X_s,Y_s\right\}$ and $\left\{X_t^p \right\}_{p=1}^P$, we aim to find a function $f$ to predict the correct label $y_i^p$ for any unlabled instance $x_i^p, i\in [1,n_t^p]$ in the $p$-th target domain, i.e., $y_i^p=f(x_i^p)$. The goal of \emph{multi-relevance transfer learning} is to alleviate the difficulty of distribution divergences between source-target domains and target-target domains by making them drawn closer in the uncovered latent subspaces so that the classifier $f$ can be trained as accurate as possible.
\subsection{Model Formulation}
To achieve the goal of multi-relevance transfer learning, we propose an algorithm which enable knowledge can be shared between source-target domains and target-target domains for sufficient transfer learning.
\subsubsection{Source-Target Knowledge Transfer}
We first discover the latent factors shared across each `source-target' pairs. It can be formulated as a collective way of nonnegative matrix tri-factorization (NMTF).
Given source domain $\mathcal{D}_s$ and the $p$-th target domain $\mathcal{D}_t^p$, one can decompose their feature-instance matrices $X_s$ and $X_t^p$ simultaneously, allowing the decomposed matrices share the cross-domain latent subspaces. Motivated by~\cite{gupta2010nonnegative}, cross-domain feature clusters can be partitioned into a common part and a domain-specific part. The common feature cluster subspace can be shared across domains, and the domain-specific ones are the remaining feature clusters in each domain. Since we have multiple source-target pairs, we decompose each $\{ X_s, X_t^p \}$ as follows:
{\small\begin{align}
\mathcal{L}_1=& \Arrowvert X_t^p - \left[ U^p,U_t^p \right]
 \left[
  \begin{array}{c}
          \Theta_{\mu}^p\\
          \Theta_{t}^p
  \end{array}
  \right]
 \left(V_t^p \right)^{\mathrm T} \Arrowvert ^2 \notag\\
            &+ \Arrowvert X_s-  \left[ U^p,U_s^p \right]
 \left[
  \begin{array}{c}
          \Theta_{\mu}^p\\
          \Theta_{s}^p
  \end{array}
  \right]
 Y_s^{ \mathrm T} \Arrowvert ^2,
%\quad  &\mbox{s.t.} \quad U^p, U_t^p, U^p_s, V_t^p, \Theta^p_{\mu}, \Theta^p_{t}, %\Theta^p_{s}, \Theta_{\mu}, \Theta_{\sigma} \geq 0, \notag\\
%\quad &  \sum \limits_{i=1}^{M} U^p_{(ij)}=1,\sum \limits_{i=1}^{M}\left( U_t^p \right)_{(ij)}=1, %\sum \limits_{i=1}^{M}\left( U_s^p \right)_{(ij)}=1, \sum \limits_{j=1}^{c} \left(V_t^p %\right)_{(ij)}=1.
\label{eq:single}
\end{align}}where $U^p \in \mathbb{R}^{M \times k_1}$ is the subspace of \emph{common} feature clusters shared across domains, $\Theta_{\mu}^p$ is its corresponding subspace of common cluster association , $U_s^p \in \mathbb{R}^{M \times (k_2-k_1)}$ and $U_t^p \in \mathbb{R}^{M \times (k_2-k_1)}$ are the subspaces of \emph{remaining} feature clusters in $\mathcal{D}_s$ and $\mathcal{D}_t^p$ respectively, $\Theta_t^p$ and $\Theta_s^p$ are the subspaces of remaining cluster association. In this way, the source domain data matrix can be decoupled for each source-target pairs by sharing different subspaces $U^p$ and $\Theta_{\mu}^p$ across domain, thus tranferring different knowledge.
\subsubsection{Multi-Relevance Transfer Learning Algorithm}
As shown in Figure 1, in MRTL, knowledge also needs to be transferred between target domains. Since no label information in targets, we need exploit their feature relatedness.
That is, we need uncover the shared common latent factors between them as bridges for mutual reinforcement.
Therefore, to capture the latent factors between targets, we formulate the factorization of feature-instance matrices of target domains by sharing subspaces as follows:
{\small\begin{align}
\mathcal{L}_2=&\sum \limits_{p=1}^P \Arrowvert X_t^p - \left[ U^p,U_t^p \right]
 \left[
  \begin{array}{c}
          \Theta_{\mu}\\
          \Theta_{\sigma}
  \end{array}
  \right]
 \left( V_t^p \right)^{ \mathrm T} \Arrowvert ^2,
\label{eq:targets}
\end{align}}where $ \Theta_{\mu} \in \mathbb{R}^{k_1 \times c}$ and $\Theta_{\sigma} \in \mathbb{R}^{(k_2-k_1) \times c}$ are the cluster association subspaces shared by target domais.
Finally, we can combine $\mathcal{L}_1$ in (\ref{eq:single}) and $\mathcal{L}_2$ in (\ref{eq:targets}) into a joint optimization formulation as follows:
{\small\begin{align}
\mathcal{L}=&\sum \limits_{p=1}^P  \Bigg\{  \Arrowvert X_t^p - \left[ U^p,U_t^p \right]
 \left[
  \begin{array}{c}
          \Theta_{\mu}^p\\
          \Theta_{t}^p
  \end{array}
  \right]
 \left(V_t^p \right)^{\mathrm T} \Arrowvert ^2 \notag\\
            &+ \Arrowvert X_s-  \left[ U^p,U_s^p \right]
 \left[
  \begin{array}{c}
          \Theta_{\mu}^p\\
          \Theta_{s}^p
  \end{array}
  \right]
 Y_s^{ \mathrm T} \Arrowvert ^2 \notag\\
      &+\lambda \Arrowvert X_t^p - \left[ U^p,U_t^p \right]
 \left[
  \begin{array}{c}
          \Theta_{\mu}\\
          \Theta_{\sigma}
  \end{array}
  \right]
 \left( V_t^p \right)^{ \mathrm T} \Arrowvert ^2 \Bigg\},
 \label{eq:objective1}
\end{align}}where $\lambda$ is the trade-off parameter weighting the contribution of target domain relatedness.
The first two terms refer to the feature clusters and label propagation between source and target domains, the third term refers to the feature clusters and label updating among target domains. Overall, the proposed learning algorithm fits the multi-relevance relationship among all the domains. As we discussed in Section 2, the decomposed matrix $U$ contains the information on hidden feature clusters, indicating the distribution of features on each hidden cluster. Therefore, the summation of each column of $U$ has to be equal to one. The label matrix $V$ indicates the label distribution of each instance. Thus, the summation of each row of $V$ has to be equal to one. Considering these constraints, we obtain the final optimization objective function
of the proposed learning algorithm:
{\small\begin{align}
\min_{\Omega \geq 0} \quad  & \mathcal{L}    \notag\\
\mbox{s.t.}\quad  &  \sum \limits_{i=1}^{M}\left( U_t^p \right)_{(ij)}=1, \sum \limits_{i=1}^{M}\left( U_s^p \right)_{(ij)}=1, \notag\\
\quad  & \sum \limits_{i=1}^{M} U^p_{(ij)}=1, \quad\sum \limits_{j=1}^{c} \left(V_t^p \right)_{(ij)}=1,
\label{eq:objective2}
\end{align}}where $\Omega=\left\{U^p, U_t^p, U^p_s, V_t^p, \Theta^p_{\mu}, \Theta^p_{t}, \Theta^p_{s}, \Theta_{\mu}, \Theta_{\sigma} \right\}$ is the parameter set. Since the objective function in  (\ref{eq:objective2}) is non-convex, it is intractable to obtain the global optimal solution. Therefore, we develop an alternating algorithm following the theory of constrained optimization~\cite{boyd2004convex}. Specifically, we optimize one variable while fixing the rest variables. The procedure repeats until convergence. \\
\indent We first show the updating rules of matrices $U_t^p$, $U_s^p$, $U^p$, and $V_t^p$ as follows:
{\small\begin{align}
\label{eq:UUUV}
U_t^p&= U_t^p \cdot \sqrt{\frac{X_t^p V_t^p \left(\Theta_t^p \right)^{ \mathrm T}+\lambda X_t^p V_t^p \Theta_{\sigma}^{\mathrm T}}{ F_{1}V_{t}^p\left(\Theta_t^p \right)^{ \mathrm T}+\lambda F_3 V_t^p \Theta_{\sigma}^{ \mathrm T}}}, \notag\\
U_s^p&=U_s^p \cdot \sqrt{\frac{X_s Y_s \left(\Theta_s^p \right)^{\mathrm T}}{F_2 Y_s \left(\Theta_s^p \right)^{\mathrm T} }}, \\
U^p&=U^p \cdot \sqrt{\frac{  X_t^p V_t^p \left(\Theta_{\mu}^p \right)^{\mathrm T} + X_s Y_s  \left(\Theta_{\mu}^p \right)^{\mathrm T}+\lambda X_t^p V_t^p \left(\Theta_{\mu} \right)^{\mathrm T}}{F_{1} V_t^p \left(\Theta_{\mu}^p\right)^{\mathrm T}  + F_{2} Y_{s} \left( \Theta_{\mu}^p\right)^{\mathrm T} + \lambda F_{3} V_t^p \left(\Theta_{\mu} \right)^{\mathrm T}}}, \notag\\
V_t^p&= V_t^p \cdot \sqrt{\frac{ \left(X_t^p \right)^{\mathrm T} \left( U^p \Theta_{\mu}^p+U_t^p \Theta_t^p \right) +  \left(X_t^p \right)^{\mathrm T} \left( U^p \Theta_{\mu}+U_t^p \Theta_{\sigma}
\right)
}{F_1^{\mathrm T}  \left( U^p \Theta_{\mu}^p+U_t^p \Theta_t^p \right) +\lambda F_3^{\mathrm T} \left( U^p \Theta_{\mu}+U_t^p \Theta_{\sigma}
\right)  }}, \notag
\end{align}}where $F_{1}=U^p \Theta_{\mu}^p (V_t^p)^{\mathrm T}+U_t^p \Theta_{t}^p (V_t^p)^{\mathrm T}$, $F_{2}=U^p\Theta_{\mu}^p Y_{s}^{\mathrm T} + U_s^{p} \Theta_{s}^p Y_{s}^{\mathrm T}$, $F_{3}=U^p \Theta_{\mu}(V_t^p)^{\mathrm T} +U_t^p \Theta_{\sigma}(V_t^p)^{\mathrm T}$, and $\cdot$ denotes matrix Hadamard product. From (\ref{eq:objective2}), after the matrices are updated, the constrained matrices have to be normalized as:
{\small\begin{align}
\left(U_t^p \right)_{(ij)}&=\frac{\left(U_t^p \right)_{(ij)}}{\sum \limits_{i=1}^{M} \left(U_t^p \right)_{(ij)}},\quad \left(U_s^p \right)_{(ij)}=\frac{\left(U_s^p \right)_{(ij)}}{\sum \limits_{i=1}^{M} \left(U_s^p \right)_{(ij)}} , \notag\\
U^p_{(ij)}&=\frac{U^p_{(ij)}}{\sum \limits_{i=1}^{M} U^p_{(ij)}}, \quad \quad  \left(V_t^p \right)_{(ij)}=\frac{\left(V_t^p \right)_{(ij)}}{\sum \limits_{j=1}^{c} \left(V_t^p \right)_{(ij)}}.
\label{eq:normalize}
\end{align}}
\noindent Similarly, the updating rules for other submatrices are:
{\small\begin{align}
\Theta_{\mu}^p &= \Theta_{\mu}^p \cdot \sqrt{\frac{\left(U^p  \right)^{\mathrm T}X_t^p V_t^p + \left(U^p  \right)^{\mathrm T} X_s Y_s}{\left(U^p  \right)^{\mathrm T} F_1 V_t^p + \left(U^p  \right)^{\mathrm T} F_2 Y_s}}, \notag\\
\Theta_{t}^p &= \Theta_{t}^p \cdot \sqrt{\frac{\left(U_t^p  \right)^{\mathrm T} X_t^p V_t^p}{\left(U_t^p  \right)^{\mathrm T}F_1 V_t^p}}, \notag\\
\Theta_{s}^p &= \Theta_{s}^p \cdot \sqrt{\frac{\left(U_s^p  \right)^{\mathrm T} X_s Y_s}{\left(U_s^p  \right)^{\mathrm T}F_2 Y_s}}, \notag\\
\Theta_{\mu}&= \Theta_{\mu} \cdot \sqrt{\frac{\sum \limits_{p=1}^P \left(U^p  \right)^{\mathrm T} X_t^p V_t^p}{\sum \limits_{p=1}^P \left(U^p  \right)^{\mathrm T}  F_3 V_t^p}}, \notag\\
\Theta_{\sigma}&=\Theta_{\sigma} \cdot \sqrt{\frac{\sum \limits_{p=1}^P \left(U_t^p  \right)^{\mathrm T} X_t^p V_t^p}{\sum \limits_{p=1}^P \left(U_t^p  \right)^{\mathrm T} F_3 V_t^p}}.
\label{othervar}
\end{align}}
\indent These lead to the procedure of the proposed MRTL algorithm in Algorithm 1. Moreover, as shown in (\ref{eq:UUUV}), $U^p$ and $U_t^p$ are constrained by $X_s$, $Y_s$, and $X_t^p$. $\Theta_{\mu}$ and $\Theta_{\sigma}$ are constrained by all the target feature matrices $\{ X_t^p\}_{p=1}^P$. Therefore, the updating rule of $V_t^p$ is constrained by $X_s$, $Y_s$, and $\{ X_t^p\}_{p=1}^P$. That is, the information in the source domain and other target domains can be transferred to the $p$-th target.

\begin{algorithm}
\SetKw{Input}{Input:}
\SetKw{Initialization}{Initialization:}
\SetKw{Output}{Output:}
\SetKwBlock{for}{for}{end}
\SetKwBlock{while}{while}{end}

\Input  $\{X_s, Y_s\}$ from source domain,  $\left\{X_t^p\right\}_{p=1}^P$  from target domains, number of target domains $P$, trade-off parameter $\lambda$, common feature clusters $k_1$, total feature clusters $k_2$, number of iterations $maxiter$.\;
\Initialization nomalize $X_s$ and  $\left\{X_t^p\right\}_{p=1}^P$, initialize $\left\{ V_t^p \right\}_{p=1}^P$ by logistic regression trained on source domain data $\left\{X_s,Y_s\right\}$.\;
iteration index $iter \leftarrow1$.\;
\while($iter<maxiter$ \textbf{do}){
\for($p=1$ \textbf{to} $P$ \textbf{do}){
 update the submatrices $U_t^p,  U_s^p, U^p,\Theta_{\mu}^p, \Theta_{t}^p, \Theta_{s}^p$, and label matrix $V_t^p$ according to the updating rules given in (\ref{eq:UUUV}) and (\ref{othervar}).\;
normalize the submatrices $U_t^p,  U_s^p, U^p$ and label matrix $V_t^p$ according to the normalization rules given in (\ref{eq:normalize}).\;
}
update the submatrices $\Theta_{\mu}$ and $\Theta_{\sigma}$ according to the updating rules in (\ref{othervar}).\;
compute objective value $\mathcal{L}^{iter}$.\;
$iter=iter+1$.\;
}
\Output the predicted results $\{V_t^p\}_{p=1}^P$
\caption{\small{MRTL: Multi-Relevance Transfer Learning}}
\label{alg:MDTL}
\end{algorithm}

\section{Theoretical Analysis}
This section aims to analyze the convergence property of the proposed algorithm. Without loss of generality, we formulate the detailed optimization updating of parameter $U_t^p$. The Lagrange function with constraint $U_t^p \geq 0$ is given by:
{\small \begin{align}
\label{eq:lagrange}
&\mathcal{L} =\sum \limits_{p=1}^P \mathrm{tr}\bigg[\left(X_t^p \right)^{\mathrm T} X_t^p -2 \left(X_t^p \right)^{\mathrm T} F_{1} + F_{1}^{\mathrm T}F_{1} + X_s^{\mathrm T}X_s -2 X_s^{\mathrm T}F_{2}  \notag\\
&  + F_{2}^{\mathrm T}F_{2}
+\lambda \left(X_t^p \right)^{\mathrm T}X_t^p  -2\lambda \left(X_t^p \right)^{\mathrm T} F_{3} +\lambda F_{3}^{\mathrm T} F_{3}\bigg] \\
&+\sum \limits_{p=1}^P \mathrm{tr}\left[\mathbf{\Lambda} \left( \left(U_t^p \right)^{\mathrm T}\textbf{1}_{M} - \textbf{1}_{(k_{2}-k_{1})}\right)\left( \left(U_t^p \right)^{\mathrm T}\textbf{1}_{M} - \textbf{1}_{(k_{2}-k_{1})}\right)^{\mathrm T}\right], \notag
%\label{eq:lagrange}
\end{align}}where $\mathbf{\Lambda} \in \mathbb{R}^{(k_{2}-k_{1}) \times (k_{2}-k_{1}) }$ is a diagonal matrix of Lagrange multiplier, $\mathbf{1}_M$ and $\mathbf{1}_{(k_2-k_1)}$ are all-ones vectors with dimension $M$ and $(k_2-k_1)$ respectively. Using the Karush-Kuhn-Tucker (KKT) complementarity condition, we have:
{\small\begin{align}
\frac{\partial \mathcal{L}}{\partial U_t^p}& \cdot U_t^p = \Big(-2X_t^p V_t^p \left(\Theta_t^p \right)^{\mathrm T} +2 F_{1} V_t^p \left(\Theta_t^p \right)^{\mathrm T}-2\lambda X_t^p V_t^p \Theta_{\sigma}^{\mathrm T} \notag\\
& +2\lambda F_3 V_t^p \Theta_{\sigma}^{\mathrm T}+2\mathbf{\Lambda} \left(U_t^p \right)^{\mathrm T}\textbf{1}_{M}\textbf{1}_{M}^{\mathrm T}-2\mathbf{\Lambda} \textbf{1}_{(k_2-k_1)}\textbf{1}_{M}^{\mathrm T}\Big)\cdot U_t^p\notag\\
&=0.
 \label{eq:KKTUtp}
\end{align}}
\begin{lemma}
\label{Lemma1}
Using the updating rule in (\ref{eq:updateUtp_KKT}) and normalization rules in (\ref{eq:normalize}), the loss function in (\ref{eq:lagrange}) will monotonously decrease until convergence.
{\small\begin{align}
U_t^p&= U_t^p \cdot \sqrt{\frac{X_t^p V_t^p \left(\Theta_t^p \right)^{ \mathrm T}+\lambda X_t^p V_t^p \Theta_{\sigma}^{\mathrm T}+\mathbf{\Lambda} \textbf{1}_{(k_2-k_1)}\textbf{1}_M^{\mathrm T}}{ F_{1}V_{t}^p\left(\Theta_t^p \right)^{ \mathrm T}+\lambda F_3 V_t^p \Theta_{\sigma}^{ \mathrm T}+\mathbf{\Lambda} \left( U_t^p\right)^{ \mathrm T} \mathbf{1}_M \mathbf{1}_M^{\mathrm T}}},
\label{eq:updateUtp_KKT}
\end{align}}
\end{lemma}
We use the auxiliary function approach~\cite{Lee00algorithmsfor} to prove Lemma~\ref{Lemma1}.
\begin{lemma}~\cite{Lee00algorithmsfor}
\label{Lemma2}
A funtion $G(Y, \widetilde{Y})$ is an auxiliary function for $\mathcal{T}(Y)$ if the conditions $G(Y, \widetilde{Y}) \geq \mathcal{T}(Y)$ and $G(Y, Y)=\mathcal{T}(Y)$
are satisfied for any $Y$, $\widetilde{Y}$. If $G$ is an auxiliary function for $\mathcal{T}$, then $\mathcal{T}$ is non-increasing under the update
{\small\begin{align}
Y^{(t+1)}=\arg\min_{Y} G\left(Y,Y^{(t)} \right).
\label{eq:nonincreasing}
\end{align}}
\end{lemma}
\begin{theorem}
\label{Theorem1}
Let $\mathcal{T}(U_t^p)$ denote the sum of all terms that contain $U_t^p$ in the loss function $\mathcal{L}$ in (\ref{eq:lagrange}) . Then the following
{\small\begin{align}
G\left(U_t^p, \widetilde{U}_t^p \right)& =-2\sum \limits_{ij}\Big( X_t^p V_t^p \left(\Theta_t^p \right)^{ \mathrm T} +\lambda X_t^p V_t^p \Theta_{\sigma}^{\mathrm T} \Lambda \mathbf{1}_{(k_2-k_1)}\mathbf{1}_{M}^{ \mathrm T} \Big)_{ij}  \notag\\
& \Big( \widetilde{U}_t^p \Big)_{ij}\Big(1+\log\frac{\left(U_t^p \right)_{ij}}{(\widetilde{U}_t^p )_{ij}}   \Big)  +\sum \limits_{ij} \Big(  F_1 V_t^p \left(\Theta_t^p \right)^{ \mathrm T} \notag\\
+&\lambda F_3 V_t^p \Theta_{\sigma}^{ \mathrm T} + \Lambda \left(U_t^p \right)^{ \mathrm T} \mathbf{1}_M \mathbf{1}_M^{ \mathrm T}\Big)_{ij}\frac{\left(U_t^p \right)_{ij}^2}{(\widetilde{U}_t^p )_{ij}}
\label{eq:constructaux}
\end{align}}is an auxiliary function for $\mathcal{T}(U_t^p)$ and is a convex function in $U_t^p$ and has a global minimum.
\end{theorem}
\indent Theorem \ref{Theorem1} can be proved similarly in~\cite{DBLP:conf/kdd/DingLPP06}. We omit the details here due to limited space. Based on Theorem \ref{Theorem1},  $G(U_t^p,\widetilde{U}_t^p)$ can be minimized with respect to $U_t^p$ and $\widetilde{U}_t^p$ fixed. Setting $\partial G(U_t^p,\widetilde{U}_t^p)/\partial U_t^p=0$ leads to the updating rule in (\ref{eq:updateUtp_KKT}). Then Lemma \ref{Lemma1} holds.
The variable $\Lambda$ in (\ref{eq:updateUtp_KKT}) still needs to be calculated. In (\ref{eq:objective2}), $\Lambda$ is used to satisfy the condition that the summation of each column of $U_t^p$ is 1. We use the normalization method (\ref{eq:normalize}) which satisfies this condition regardless of $\Lambda$. Then, $\mathbf{\Lambda} \left(U_t^p \right)^{\mathrm T}\textbf{1}_{M}\textbf{1}_{M}^{\mathrm T}$  is equal to $\mathbf{\Lambda} \textbf{1}_{(k_2-k_1)}\textbf{1}_{M}^{\mathrm T}$. Hence, (\ref{eq:UUUV}) is an approximation to (\ref{eq:updateUtp_KKT}).
\begin{theorem}
\label{Theorem2}
Using Algorithm 1 to update $U_t^p$, $\mathcal{T}\left( U_t^p\right)$ will monotonically decreases.
\end{theorem}
\begin{proof}
By Lemma \ref{Lemma2} and Theorem \ref{Theorem1}, we have $\mathcal{T}\left((U_t^p)^0 \right)=$\\
$  G\left((U_t^p)^0, (U_t^p)^0\right) \geq G\left((U_t^p)^1, (U_t^p)^0\right) \geq G\left((U_t^p)^1, (U_t^p)^1\right)=\mathcal{T}\left((U_t^p)^1 \right) \geq ...$ Therefore $\mathcal{T}\left( U_t^p\right)$ is monotonically decreasing.
\end{proof}
Theorem \ref{Theorem2} also hold water with respect to the other variables. Since the objective function $\mathcal{L}$ is obviously lower bounded by 0, Algorithm 1 is guaranteed to converge.

\section{Experiments}
Experiments are tested on two benchmark data
sets: 20-Newsgroups and Email Spam data sets, which are widely
adopted for transfer learning evaluation.\\
\indent \textbf{20-Newsgroups} The 20 newsgroups dataset\footnote{http://people.csail.mit.edu/jrennie/20Newsgroups/} contains 18,774 documents, and has
\begin{table}[htbp]
\small
  \centering
  \caption{Data sets generated from 20 Newsgroups}

    \begin{tabular}{c|c|c}
    %\toprule
    \hline
    \textbf{Data Set} & Source Domain & Target Domain \\ \hline \hline
   % \midrule
       &       & comp.graphics \\
          &       & rec.autos \\ \cline{3-3}
    comp vs. rec & comp.sys.mac.hardware & comp.os.ms-windows.misc \\
          & rec.sport.hockey & rec.motorcycles \\ \cline{3-3}
          &       & comp.sys.ibm.pc.hardware \\
          &       & rec.sport.baseball \\ \hline
          &       & rec.autos \\
          &       & sci.crypt \\ \cline{3-3}
    rec vs. sci & rec.sport.hockey & rec.motorcycles \\
          & sci.space & sci.electronics \\ \cline{3-3}
          &       & rec.sport.baseball \\
          &       & sci.med \\ \hline
                    &       & sci.crypt \\
          &       & comp.graphics\\ \cline{3-3}
    sci vs. comp & sci.space & sci.electronics \\
          & comp.sys.mac.hardware & comp.os.ms-windows.misc \\ \cline{3-3}
          &       & sci.med \\
          &       & comp.sys.ibm.pc.hardware \\  \hline
    %\bottomrule
    \end{tabular}%
  \label{tab:addlabel}%
\end{table}%
a hierarchical structure with 6 main categories and 20 subcategories. Following ~\cite{DuanTNNLS2012}, we
choose the instances from three main categories \emph{comp}, \emph{rec}, \emph{sci}, with at least four subcategories to generate three settings to evaluate
our proposed algorithms. For each setting, we choose one main category as the positive class
and use another one as the negative class, and employ all the labeled instances from two subcategories to construct one domain. In the experiments, we construct one
source domain and three target domains (see Table 1 for details).\\
\indent \textbf{Email Spam} The email spam dataset\footnote{http://www.ecmlpkdd2006.org/challenge.html} contains 4000 publicly available labeled emails as well as three email sets (each contains 2500 emails) annotated by three different users. Therefore, the distributions of the publicly available email set and three user-annotated email sets differ from each other. For each set, a half of the emails are non-spam (labeled as 1) and the others are spam (labeled as -1). We
consider the publicly available email set as the source domain and
the three user-annotated sets as three target domains.

\subsection{Experimental Setup}
\begin{table*}[t]
\tiny
  \centering
  \caption{Average Classification Accuracy ($\%$) on 20 Newsgroups Dataset}

    \begin{tabular}{c|c|c|c|c|c|c|c|c|c|c|c|c}
    \hline
    Data set & Target & NMF   & LG    & SVM   & TSVM  & MTFL  & MTrick & $\text{DTL}_0$ & $\text{DTL}_1$  & $\text{TriTL}_0$ & $\text{TriTL}_1$ & MRTL \\ \hline
          & target-1 & 63.46 & 60.60  & 58.35 & 87.89 & 60.04 & 94.17 & 93.20 & 93.51  & 93.82 & 90.60  & \textbf{95.09} \\ \cline{2-13}
    comp vs. rec & target-2 & 59.71 & 65.64 & 64.78 & 92.38 & 66.58 & 93.71 & 93.97 &  96.01 & 94.94 & 92.33 & \textbf{98.26} \\ \cline{2-13}
          & target-3 & 58.98 & 92.44 & 92.99 & 95.63 & 93.31 & 97.26 & 97.21 &  \textbf{97.87} & \textbf{97.82} & 97.67 & \textbf{97.97} \\ \cline{1-13}
    \multicolumn{2}{c|}{Average}       & 60.72 & 72.89 & 72.04 & 91.97 & 73.31 & 95.05 & 94.79 & 95.80 & 95.52 & 93.53 & \textbf{97.11} \\ \cline{1-13}
          & target-1 & 52.58 & 55.31 & 53.85 & 87.60  & 59.41 & 90.08 & 88.41 & 88.41 & 90.03 & 84.16 & \textbf{90.74} \\ \cline{2-13}
    rec vs. sci & target-2 & 50.28 & 57.16 & 56.45 & 83.76 & 60.58 & 91.35 & 88.77 & 90.95 & 86.80  & 89.58 & \textbf{92.46} \\ \cline{2-13}
          & target-3 & 63.50  & 85.84 & 86.86 & 92.32 & 87.45 & 96.71 & \textbf{97.37} & \textbf{97.47} & 95.75 & 96.01 & 96.97 \\ \cline{1-13}
   \multicolumn{2}{c|}{Average}       & 55.45 & 66.10  & 65.72 & 87.89 & 69.15 & 92.72 & 91.52 & 92.28 & 90.86 & 89.92 & \textbf{93.39} \\ \cline{1-13}
          & target-1 & 67.79 & 70.34 & 67.79 & 82.34 & 67.13 & 87.29 & 88.21 & 87.95  &  89.84 & 88.82 & \textbf{90.30} \\ \cline{2-13}
    sci vs. comp & target-2 & 57.47 & 60.35 & 59.84 & 66.46 & 59.22 & 75.04 & 76.22 & 76.07 & 76.12 & 74.27 & \textbf{80.07} \\ \cline{2-13}
          & target-3 & 53.36 & 81.94 & 81.59 & 91.05 & 79.86 & 98.02 & \textbf{98.27} & 97.41  & 97.81 & 97.20  & \textbf{98.58} \\ \cline{1-13}
    \multicolumn{2}{c|}{Average}        & 59.54 & 70.88 & 69.74 & 79.95 & 68.74 & 86.78 & 87.57 & 87.14  & 87.92 & 86.76 & \textbf{89.65} \\
    \hline
    \end{tabular}%
  \label{tab:addlabel}%
\end{table*}%

\begin{table*}[t]
\tiny
  \centering
  \caption{Average Classification Accuracy  ($\%$)  on Email Spam Dataset}
    \begin{tabular}{c|c|c|c|c|c|c|c|c|c|c|c|c}
    \hline
    Source & Target & NMF   & LG    & SVM   & TSVM  & MTFL  & Mtrick & $\text{DTL}_0$  & $\text{DTL}_1$  & $\text{TriTL}_0$ & $\text{TriTL}_1$ & MRTL \\
    \hline
          & User 1 & 73.36 & 65.56 & 56.36 & 72.64 & 58.92 & 83.16 & 82.92 & 82.04 & 81.80  & 79.16 & \textbf{83.48} \\  \cline{2-13}
    Public Set & User 2 & 77.80  & 67.28 & 61.32 & 77.92 & 62.84 & 84.36 & 84.04 & 84.52 &85.16 & 78.76 & \textbf{86.68} \\  \cline{2-13}
          & User 3 & 79.16 & 81.84 & 69.32 & 90.64 & 70.08 & 90.36 & 90.40 & 91.08 & 92.04 & \textbf{92.68} & \textbf{92.48} \\ \cline{1-13}
   \multicolumn{2}{c|}{Average}     & 76.77 & 71.56 & 62.33 & 80.39 & 63.95 & 85.96 & 85.79 & 86.07 & 86.33 & 83.53 & \textbf{87.55} \\
    \hline
    \end{tabular}%
  \label{tab:addlabel}%
\end{table*}%
We compare the proposed MRTL with several state-of-the-art methods: (1) Unsupervised method Nonnegative Matrix Factorization (NMF) ~\cite{Lee00algorithmsfor}, which is directly applied to the target domain data. (2) Supervised methods, including Logistic Regression\footnote{http://research.microsoft.com/en-us/um/people/minka/\\papers/logreg/} (LG) and Support Vector Machine (SVM), which are trained on the source domain data and tested on the target domain data using the implementation
in LibSVM\footnote{http://www.csie.ntu.edu.tw/~cjlin/libsvm/} with linear kernel SVM. (3) Semi-supervised learning method Transductive Support
Vector Machine\footnote{http://www.cs.cornell.edu/People/tj/svm\_light/}(TSVM)~\cite{joachims1999transductive}, which works in a transductive setting using both source and target domain data for training. (4) Multi-task learning method Multi-Task Feature Learning (MTFL)~\cite{evgeniou2007multi}. It is trained on the source domain and tested on all the target domains simultaneously. (5) The state-of-the-art transfer learning methods, including Matrix Tri-Factorization based Classification (MTrick)~\cite{DBLP:conf/sdm/ZhuangLXHXS10}, Dual Transfer Learning (DTL)~\cite{DBLP:conf/sdm/LongWDCZW12} and Triplex Transfer Learning (TriTL)~\cite{zhuang2014triplex}. Both DTL and TriTL can be extended to solve multiple target domains by making the source and all the target domains share the same feature cluster subspace. In the experiments, the single target ones are referred to as $\text{DTL}_0$ and $\text{TriTL}_0$, and the extension ones are $\text{DTL}_1$ and $\text{TriTL}_1$, respectively. $\text{TriTL}_1$ and MRTL are trained using the source domain and all target domains.\\
\indent Since model selection is still an open question in transductive transfer learning, one practical solution is to choose one existing labeled data set to make training and validation.  Therefore, we select \emph{comp vs.rec} to conduct corss-validation. The parameters of the proposed method and baselines are tuned on the data set \emph{comp vs.rec}. Then the tuned parameters are applied to all other data sets. The parameters of MRTL include the trade-off parameter $\lambda$, the number of common feature clusters $k_1$ and total feature clusters $k_2$. In the comparison experiments (see Table 2 and 3), we set $k_2=50$, $k_1=10$, $\lambda=10$, $maxiter=100$.
\subsection{Experimental Results and Discussion}
Table 2 and 3 show the accuracy of all these algorithms on each target domain and their average. We can observe from the results that the proposed MRTL consistently outperforms the considered rivals on each data set. We can also find that the non-transfer methods NMF, LG, and SVM cannot perform well on most data sets. MTFL performs poorly because without transfer the multitask classifiers trained on the source domain cannot discriminate well on target domains.
TSVM outperforms them on many data sets which verifies the unlabeled data can help improve performance, but performs worse when the distribution diversity across domain is large. MTrick, $\text{DTL}_0$ and $\text{TriTL}_0$ performs better than the non-transfer methods, but have not reached the best performances becauese of the restriction that they fail to exploit the relatedness between target domains. $\text{TriTL}_1$ cannot simultaneously perform well on all target domains since it assumes that all the `source-target' pairs share the same latent subspaces which would lead to insufficient transfer or overfitting on some target domains.\\
\indent To verify that exploiting the relatedness between target domains indeed brings about effectiveness, MRTL is compared with $\text{TriTL}_1$, $\text{TriTL}_0$, $\text{DTL}_0$, $\text{DTL}_1$, and MTrick. We plot the average classification performance of MRTL with respect to $\lambda$ on \emph{comp vs. rec} data set in Figure \ref{fig:changelambda}. The average baseline results are shown as
dashed lines. It can be seen that the performance of MRTL improves at first with the increasing of $\lambda$.  When the parameter varies in a wide range $\lambda \in [1, 100]$, MRTL performs quite stably and consistently outperforms the baselines. It indicates that by exploiting the relatedness between target domains, MRTL achieves optimal transferability. Also, we test the model paramete $k_1$ varying from 5 to 50 to analyze how it affects the average classification performance. The results are shown in Figure \ref{fig:changek1}, from which we can see that the average accuracy increases at first and then decreases, which indicates that only a part of feature clusters are shared as common, thus the partition of feature cluster subspaces is justified.

{\pgfplotsset{footnotesize,samples=5}
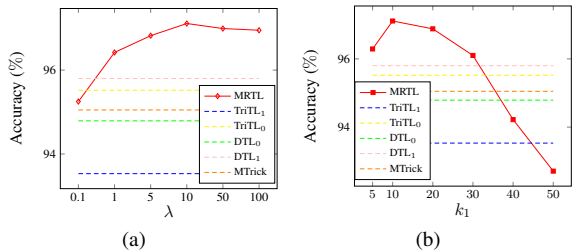
\begin{figure}[htbp]
\begin{center}
\subfigure[]{
\tikzset{every mark/.append style={scale=1.1}}
\begin{tikzpicture}[scale=0.58]
\begin{axis}[
height = 5.7cm,
width = 6.5cm,
grid=major,
xlabel style = {font=\large,yshift = 7pt},
xlabel = {$\lambda$},
ylabel style = {font=\large,yshift = -7pt},
ylabel = {Accuracy ($\%$)},
legend style={at={(0.9905,0.5745)},font=\scriptsize},
legend cell align = left,
ymajorgrids=false,
xmajorgrids=false,
ytick = {0.90,0.92,0.94,0.96,0.98},
xtick ={1,2,3,4,5,6},
xticklabels={0.1,1,5,10,50,100},
yticklabels={90,92,94,96,98},
cycle multi list={
{red,mark=diamond,mark size=1.5pt},
{blue,mark=, densely dashed,mark size=3pt},
{yellow,mark=, densely dashed,mark size=3pt},
{green,mark=, densely dashed,mark size=3pt},
{pink,mark=, densely dashed,mark size=3pt},
{orange,mark=, densely dashed,mark size=3pt},
}
]

\addplot coordinates {
%(1,0.914)
%(1,0.9525)
%%(2,0.9232)
%(3,0.9443)
%(4,0.9642)
%(5,0.9682)
%(6,0.9711)
%(7,0.9699)
%(8,0.9695)
(1,0.9525)
(2,0.9642)
(3,0.9682)
(4,0.9711)
(5,0.9699)
(6,0.9695)
};
\addplot coordinates {
(1,0.935314)
(2,0.935314)
(3,0.935314)
(4,0.935314)
(5,0.935314)
(6,0.935314)
%(7,0.935314)
%(8,0.935314)
};
\addplot coordinates {
(1,0.9552)
(2,0.9552)
(3,0.9552)
(4,0.9552)
(5,0.9552)
(6,0.9552)
%(7,0.9552)
%(8,0.9552)
};
\addplot coordinates {
(1,0.9479)
(2,0.9479)
(3,0.9479)
(4,0.9479)
(5,0.9479)
(6,0.9479)
%(7,0.9479)
%(8,0.9479)
};
\addplot coordinates {
(1,0.958)
(2,0.958)
(3,0.958)
(4,0.958)
(5,0.958)
(6,0.958)
%(7,0.958)
%(8,0.958)
};
\addplot coordinates {
(1,0.9505)
(2,0.9505)
(3,0.9505)
(4,0.9505)
(5,0.9505)
(6,0.9505)
%(7,0.9505)
%(8,0.9505)
};
\legend{MRTL,$\text{TriTL}_1$,$\text{TriTL}_0$,$\text{DTL}_0$,$\text{DTL}_1$,MTrick}
\end{axis}
\end{tikzpicture}

\label{fig:changelambda}
}
\subfigure[]{
\tikzset{every mark/.append style={scale=1.1}}
\begin{tikzpicture}[scale=0.58]
\begin{axis}[
height = 5.7cm,
width = 6.5cm,
grid=major,
xlabel style = {font=\large,yshift = 7pt},
xlabel = {$k_1$},
ylabel style = {font=\large,yshift = -7pt},
ylabel = {Accuracy ($\%$)},
legend style={at={(0.3505,0.5805)},font=\scriptsize},
legend cell align = left,
ymajorgrids=false,
xmajorgrids=false,
ytick = {0.92,0.94,0.96,0.98},
xtick ={5,10,20,30,40,50},
yticklabels={92,94,96,98},
cycle multi list={
{red,mark=square*,mark size=1.2pt},
{blue,mark=, densely dashed,mark size=3pt},
{yellow,mark=, densely dashed,mark size=3pt},
{green,mark=, densely dashed,mark size=3pt},
{pink,mark=, densely dashed,mark size=3pt},
{orange,mark=, densely dashed,mark size=3pt},
}
]

\addplot coordinates {
(5,0.9629)
(10,0.9711)
(20,0.9688)
(30,0.9610)
(40,0.9422)
(50,0.9271)
};
\addplot coordinates {
(5,0.935314)
(10,0.935314)
(20,0.935314)
(30,0.935314)
(40,0.935314)
(50,0.935314)
};
\addplot coordinates {
(5,0.9552)
(10,0.9552)
(20,0.9552)
(30,0.9552)
(40,0.9552)
(50,0.9552)
};
\addplot coordinates {
(5,0.9479)
(10,0.9479)
(20,0.9479)
(30,0.9479)
(40,0.9479)
(50,0.9479)
};
\addplot coordinates {
(5,0.958)
(10,0.958)
(20,0.958)
(30,0.958)
(40,0.958)
(50,0.958)
};
\addplot coordinates {
(5,0.9505)
(10,0.9505)
(20,0.9505)
(30,0.9505)
(40,0.9505)
(50,0.9505)
};
\legend{MRTL,$\text{TriTL}_1$,$\text{TriTL}_0$,$\text{DTL}_0$,$\text{DTL}_1$,MTrick}
\end{axis}
\end{tikzpicture}
\label{fig:changek1}
}

\caption{Performance of MRTL with respect to $\lambda$ and $k_1$ on \emph{comp vs. rec}  data set.}
\end{center}
\end{figure}}

The proposed method achieves better performance when $k_1$ is between 5 and 30. For different `source-target' pairs, we can set different $k_1$ values. In this paper, we simply set them equal.

\subsection{Algorithm Convergence}

In Section 4, we have theoretically proven the convergence property of the proposed MRTL algorithm. Here we empirically check the convergence by testing it on \emph{comp vs. rec} data set. In Figure \ref{fig:convergenceObj}, we show the logarithmic objective value with respect to the number of iterations. We see that after around five iterations, the objective value experiences almost no change. Similarly, we show the average classification accuracy of all target domains with respect to the number of iterations in Figure \ref{fig:convergenceAcc}. The results show that the average accuracy of MRTL increases with more iterations and converges after 80 iterations.
{
\pgfplotsset{footnotesize,samples=5}
\begin{figure}[htbp]

\setlength{\abovecaptionskip}{-1pt}
\setlength{\belowcaptionskip}{-1pt}
\begin{center}
\subfigure[{Objective values}]{
\tikzset{every mark/.append style={scale=1.1}}
\begin{tikzpicture}[scale=0.58]
\begin{axis}[
height = 5.7cm,
width = 6.54cm,
grid=major,
xlabel style = {font=\large,yshift = 7pt},
xlabel = {Iteration},
ylabel style = {font=\large,yshift = -7pt},
ylabel = {Objective value},
legend style={at={(1.0010,0.2030)}},
legend cell align = left,
ymajorgrids=false,
xmajorgrids=false,
ytick = {-3,-2,-1,0,1,2,3,4},
xtick ={1,20,40,60,80,100},
xticklabels={1,20,40,60,80,100},
cycle multi list={
{blue,mark=*,mark size=0.8pt},
}
]

\addplot coordinates {

(1,3.7716)
(2,-0.7888)
(3,-2.1696)
(4,-2.2318)
(5,-2.2354)
(6,-2.2359)
(7,-2.236)
(8,-2.236)
(9,-2.236)
(10,-2.236)
(11,-2.236)
(12,-2.236)
(13,-2.236)
(14,-2.236)
(15,-2.236)
(16,-2.236)
(17,-2.236)
(18,-2.236)
(19,-2.236)
(20,-2.236)
(21,-2.236)
(22,-2.236)
(23,-2.236)
(24,-2.236)
(25,-2.2361)
(26,-2.2361)
(27,-2.2361)
(28,-2.2361)
(29,-2.2361)
(30,-2.2361)
(31,-2.2361)
(32,-2.2361)
(33,-2.2361)
(34,-2.2361)
(35,-2.2361)
(36,-2.2361)
(37,-2.2362)
(38,-2.2362)
(39,-2.2362)
(40,-2.2362)
(41,-2.2362)
(42,-2.2362)
(43,-2.2362)
(44,-2.2362)
(45,-2.2362)
(46,-2.2363)
(47,-2.2363)
(48,-2.2363)
(49,-2.2363)
(50,-2.2363)
(51,-2.2363)
(52,-2.2363)
(53,-2.2363)
(54,-2.2363)
(55,-2.2363)
(56,-2.2363)
(57,-2.2363)
(58,-2.2363)
(59,-2.2363)
(60,-2.2363)
(61,-2.2363)
(62,-2.2363)
(63,-2.2363)
(64,-2.2363)
(65,-2.2363)
(66,-2.2363)
(67,-2.2363)
(68,-2.2363)
(69,-2.2363)
(70,-2.2363)
(71,-2.2363)
(72,-2.2363)
(73,-2.2363)
(74,-2.2363)
(75,-2.2363)
(76,-2.2363)
(77,-2.2363)
(78,-2.2363)
(79,-2.2363)
(80,-2.2363)
(81,-2.2363)
(82,-2.2363)
(83,-2.2363)
(84,-2.2363)
(85,-2.2363)
(86,-2.2363)
(87,-2.2363)
(88,-2.2363)
(89,-2.2363)
(90,-2.2363)
(91,-2.2363)
(92,-2.2363)
(93,-2.2363)
(94,-2.2363)
(95,-2.2363)
(96,-2.2363)
(97,-2.2363)
(98,-2.2363)
(99,-2.2363)
(100,-2.2363)
};

%\legend{SVM-MC ($F1^{a}$),SVM-MC ($F1^{i}$)}
\end{axis}
\end{tikzpicture}
\label{fig:convergenceObj}}
~~~~\subfigure[{Accuracy}]
{
\tikzset{every mark/.append style={scale=1.1}}
\begin{tikzpicture}[scale=0.58]
\begin{axis}[
height = 5.7cm,
width =6.54cm,
grid=major,
xlabel style = {font=\large,yshift = 7pt},
xlabel = {Iteration},
ylabel style = {font=\large,yshift = -7pt},
ylabel = {Accuracy ($\%$)},
legend style={at={(1.0005,0.2030)}},
legend cell align = left,
ymajorgrids=false,
xmajorgrids=false,
ytick = {0.70,0.75,0.80,0.85,0.90,0.95,1.00},
xtick ={1,20,40,60,80,100},
xticklabels={1,20,40,60,80,100},
yticklabels = {70,75,80,85,90,95,100},
%xtick = %{1,2,3,4,5,6,7,8,9,10,11,12,13,14,15,16,17,18,19,20,21,22,23,24,25,26,27,28,29,30,31,32,33,34,35,36,37,38,39,40,41,42,43,44,45,46,47,48,49,50,51,52,53,54,55,56,57,58,59,60,61,62,63,64,65,66,67,68,69,70,71,72,73,74,75,76,77,78,79,80,81,82,83,84,85,86,87,88,89,90,91,92,93,94,95,96,97,98,99,100},
%xticklabels = {1,,,,,,,,,,,,,,,,,,,,,,,,,,,,,,,,,,,,,,,,,,,,,,,,,50,,,,,,,,,,,,,,,,,,,,,,,,,,,,,,,,,,,,,,,,,,,,,,,,,,100},
%%%%%%%%%%cycle multi list={
%%%%%%%%%%	{blue,mark=*},
%%%%%%%%%%	{red,mark=square},
%%%%%%%%%%	{cyan,mark=square*},
%%%%%%%%%%	{olive,mark=o},
%%%%%%%%%%	{green,mark=pentagon*},
%%%%%%%%%%      {purple,mark=square},
%%%%%%%%%%}
cycle multi list={
	{red, mark=square,mark size=0.8pt},
	{red,mark=square},
%%%%%%%%%%	{blue,mark=*,densely dashed},
%%%%%%%%%%	{red,mark=square,densely dashed},
%%%%%%%%%%	{cyan!90!black,mark=diamond*},
%%%%%%%%%%      {violet!50!red,mark=o},
}
]

\addplot coordinates {
(1,0.728946)
(2,0.728946)
(3,0.728946)
(4,0.728946)
(5,0.728946)
(6,0.728946)
(7,0.728775)
(8,0.728775)
(9,0.728946)
(10,0.729116)
(11,0.729116)
(12,0.729285)
(13,0.729115)
(14,0.729115)
(15,0.729115)
(16,0.729115)
(17,0.729284)
(18,0.729284)
(19,0.729284)
(20,0.728943)
(21,0.729453)
(22,0.729453)
(23,0.729282)
(24,0.729623)
(25,0.731491)
(26,0.732511)
(27,0.734209)
(28,0.735398)
(29,0.737271)
(30,0.739988)
(31,0.744068)
(32,0.74781)
(33,0.752402)
(34,0.758183)
(35,0.762783)
(36,0.769425)
(37,0.776233)
(38,0.784404)
(39,0.793428)
(40,0.800582)
(41,0.810968)
(42,0.823911)
(43,0.833112)
(44,0.841799)
(45,0.85117)
(46,0.859855)
(47,0.867185)
(48,0.872466)
(49,0.88047)
(50,0.88541)
(51,0.892224)
(52,0.898699)
(53,0.903468)
(54,0.908919)
(55,0.913689)
(56,0.918289)
(57,0.922205)
(58,0.927144)
(59,0.92987)
(60,0.933787)
(61,0.936683)
(62,0.938898)
(63,0.941283)
(64,0.943837)
(65,0.946393)
(66,0.949628)
(67,0.95048)
(68,0.952525)
(69,0.95559)
(70,0.956101)
(71,0.957464)
(72,0.958486)
(73,0.959338)
(74,0.959849)
(75,0.960872)
(76,0.961383)
(77,0.961894)
(78,0.962405)
(79,0.963426)
(80,0.964619)
(81,0.964958)
(82,0.965639)
(83,0.966151)
(84,0.966832)
(85,0.967173)
(86,0.968024)
(87,0.968194)
(88,0.968876)
(89,0.969557)
(90,0.969897)
(91,0.969898)
(92,0.969728)
(93,0.970749)
(94,0.970579)
(95,0.970579)
(96,0.971427)
(97,0.971597)
(98,0.971427)
(99,0.971256)
(100,0.971086)

};

%\legend{dsa}
\end{axis}
\end{tikzpicture}
\label{fig:convergenceAcc}

}

\caption{Performance of MRTL with respect to iterations on \emph{comp vs. rec}  data set.}
\end{center}
\end{figure}
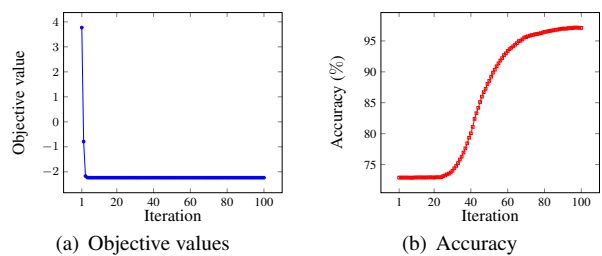
}

\section{Conclusion}
In this paper, we study multi-relevance transfer learning, where knowledge not only needs to be transferred from the source domain but also from all the target domains. We propose a MRTL framework to solve this problem. The framework achieves both source-target transfer and target-target transfer by sharing multiple decomposed latent subspaces. We develop an alternating scheme for optimization. Experiments on two datasets show the effectiveness of the proposed approach. The convergence property has also been theoretically and experimentally proven. In future, extending MRTL to tackle online tasks is an interesting problem.

\bibliography{aaai18}
\end{document}